\documentclass[11pt]{article}

\usepackage{cite}
\usepackage{amsmath,amsthm,amssymb,amsfonts}
\usepackage{bbm}
\usepackage{upgreek}
\usepackage{color}
\usepackage{dsfont}
\usepackage{authblk}
\usepackage{setspace}
\usepackage{fullpage,etoolbox}

\theoremstyle{plain}
\newtheorem{theorem}{\textbf{Theorem}}
\newtheorem{corollary}{\textbf{Corollary}}
\newtheorem{definition}{\textbf{Definition}}
\newtheorem{proposition}{\textbf{Proposition}}
\newtheorem{example}{\textbf{Example}}

\title{
Single Trajectory Conformal Prediction}
\author{Brian Lee and Nikolai Matni% <-this % stops a space
\thanks{This work was supported in part by the Penn Summer Program for Undergraduate Research (SPUR) Grant and by NSF award ECCS-2231349, SLES-2331880, and CAREER-2045834. The authors are with the University of Pennsylvania,
        Philadelphia, PA 19104, USA
        {\tt\{wblee,nmatni\}@seas.upenn.edu.}%
}}

\begin{document}
        
\maketitle

\begin{abstract}

We study the performance of risk-controlling prediction sets (RCPS), an empirical risk minimization-based formulation of conformal prediction, with a single trajectory of temporally correlated data from an unknown stochastic dynamical system. First, we use the blocking technique to show that RCPS attains performance guarantees similar to those enjoyed in the iid setting whenever data is generated by asymptotically stationary and contractive dynamics. Next, we use the decoupling technique to characterize the graceful degradation in RCPS guarantees when the data generating process deviates from stationarity and contractivity. We conclude by discussing how these tools could be used toward a unified analysis of online and offline conformal prediction algorithms, which are currently treated with very different tools. 
\end{abstract}

%%%%%%%%%%%%%%%%%%%%%%%%%%%%%%%%%%%%%%%%%%%%%%%%%%%%%%%%%%%%%%%%%%%%%%%%%%%%%%%%
\section{Introduction}

Quantifying the uncertainty associated with a learned model of a dynamical system is a key step in safely bridging learning and control. A natural approach is characterizing the finite-sample estimation error of the learned model: fix a training algorithm and sample size, then provide high probability bounds on, e.g., the $\ell_2$-distance between the learned and true system parameters \cite{simchowitz_learning_2018, sarkar_near_2019, ziemann_learning_2022}. Such finite-sample error bounds, or equivalently, confidence sets, have been used as part of learning-based control algorithms that enjoy safety and robustness guarantees. See the recent survey~\cite{tsiamis2023statistical} for an overview of such results in the linear systems setting. However, producing sharp estimation error bounds requires strong assumptions, like linear realizability, and cannot be easily applied to overparameterized and misspecified models found throughout modern machine learning. 

An alternative to confidence sets over model parameters is prediction sets over model outputs that contain the true label with high probability. Conformal prediction (CP) is a family of assumption-light methods that post-process an arbitrary model's point predictions into prediction sets that enjoy finite-sample coverage guarantees \cite{lei_distribution-free_2014, lei_distribution-free_2018, bates_distribution-free_2021}. But many CP variants' proofs of validity\footnote{Two types of CP guarantees. For error tolerance $\varepsilon$, \textit{validity} ensures that the true error is at most $\varepsilon$, and \textit{efficiency} ensures that the true error is at least $\varepsilon - o(1)$. Efficiency ensures that prediction sets are not too conservative, e.g., do not return every possible label. We characterize the effect of temporally correlated data on validity, but our techniques can be used alongside well-known tools \cite{park_pac_2020, sharma_pac-bayes_2023} to study efficiency. We defer details to future work.} require training and test data to be exchangeable or independent and identically distributed (iid), limiting applicability to temporally correlated data found in control and robotics tasks.

Existing works that apply CP to dynamical systems overcome this issue by sampling $N$ iid-initialized trajectories of length $T$ each from a system of interest, then leveraging the fact that under mild regularity conditions, for any fixed timestep $t \in \{1, 2, ..., T \}$, the set of observations aggregated across trajectories, $\left\{ x_t^{(1)}, ..., x_t^{(N)} \right\}$, is iid \cite{cleaveland_conformal_2023, hashemi_data-driven_2023, lindemann_safe_2023,tumu_physics_2023}. While practical, this multiple trajectories approach falls short of characterizing what makes conformal prediction with dynamical systems easy or hard, and does not cover data-scarce settings where, e.g., a robot must take safe actions using a single trajectory of data. Multiple trajectories with non-iid initializations and varying lengths also cannot be readily reduced to the iid data setting, further motivating a single trajectory approach.

In this work, we show that approximately valid CP---i.e., test-time error $\varepsilon + \gamma$ where $\gamma$ intuitively captures the degree of dependence in data---can be performed with a single trajectory of temporally correlated data, attaining non-trivial guarantees with improved sample efficiency. To the best of our knowledge, the additional $\gamma$ term is unavoidable and is incurred in all known results for CP with dependent data \cite{chernozhukov_exact_2018, xu_conformal_2023, barber_conformal_2023} unless strong structural assumptions are imposed~\cite{nettasinghe_extending_2023, nair_randomization_2023}. 

Our contributions are as follows.
\begin{itemize}
    \item In Section~\ref{sec:blocked-rcps}, we use the blocking technique \cite{yu_rates_1994} to show that risk-controlling prediction sets (RCPS) \cite{bates_distribution-free_2021}, an empirical risk minimization-based formulation of CP, attain iid-like guarantees whenever data is generated by asymptotically stationary and contractive dynamics. See Theorem~\ref{thm:blocked-rcps} and Corollary~\ref{cor:blocked-rcps}.
    
    \item In Section~\ref{sec:decoupled-rcps}, we use the decoupling technique \cite{pena_decoupling_1999} to characterize how the RCPS coverage guarantee degrades gracefully when the data generating process deviates from stationarity and contractivity. See Theorem~\ref{thm:decoupled-rcps}. 
\end{itemize}

\subsection{Related Work}
\paragraph{Conformal prediction with dependent data.} Early conformal prediction works assume that data is exchangeable, i.e., the joint distribution of a dataset is invariant to permutations of the samples, or more restrictively, iid~\cite{gammerman_algorithmic_2005, lei_distribution-free_2014, lei_distribution-free_2018, bates_distribution-free_2021}. However, samples in a single trajectory drawn from a dynamical system are neither exchangeable nor iid. This motivates the reduction from multiple iid-initialized trajectories to iid data, which has been used toward uncertainty-aware model predictive control \cite{lindemann_safe_2023}, physics-constrained motion prediction \cite{tumu_physics_2023}, and data-driven reachability analysis \cite{hashemi_data-driven_2023}. Recent works build toward the single trajectory setting, either by assuming that certain subsamples of dependent data are approximately independent \cite{chernozhukov_exact_2018, oliveira_split_2022, nettasinghe_extending_2023}, or by relying on entirely different algorithmic tools from online learning \cite{gibbs_adaptive_2021, gibbs_conformal_2023, bhatnagar_improved_2023}. These methods have not been widely adopted for control applications because they often make assumptions about data that are incompatible with dynamical systems of interest. For example, many existing methods assume that data is generated by a strictly stationary process \cite{chernozhukov_exact_2018, oliveira_split_2022}, a requirement that even strictly stable linear systems with iid Gaussian noise do not satisfy. Many methods also require auxiliary assumptions, e.g., the user must consistently estimate certain system parameters \cite{xu_conformal_2023}, that can be difficult to satisfy in practice. In this work, we relax both of these types of assumptions.

\paragraph{Statistical learning with dependent data.}
Two parallel efforts to relax assumptions about iid data have taken place in statistical learning, with the goal of proving generalization and estimation error bounds with temporally correlated data. First, a line of work \cite{kuznetsov_generalization_2017, ziemann_single_2022} handles a single trajectory of strictly stationary and $\beta$-mixing data (see Definition 1) using the blocking technique \cite{yu_rates_1994}. Notably, mixing is a sufficient condition for errors to converge to zero as sample size grows, albeit often at a deflated rate compared to the iid setting. Second, a line of work \cite{kuznetsov_time_2016, kuznetsov_discrepancy-based_2020} handles a single trajectory of adapted data (see Definition~\ref{def:adapted}) using the more general decoupling technique \cite{pena_decoupling_1999}. The resulting errors do not necessarily decay to zero, but to a discrepancy term that can be estimated from data and minimized with parameter tuning. We extend prior work \cite{oliveira_split_2022} in applying the blocking technique to CP, and introduce the decoupling technique to CP. 

Existing estimation error bounds, or equivalently, confidence sets in the space of models, can also be used to build prediction sets. However, such prediction sets are often excessively conservative because estimation error bounds aim for uniform convergence over model classes, i.e., they scale with complexity measures like VC dimension and metric entropy that can be loose for realistic models \cite{ziemann_single_2022}. Our generalization bound-like results in Sections~\ref{sec:blocked-rcps} and \ref{sec:decoupled-rcps} avoid this looseness because RCPS, and CP more generally, focuses on pointwise convergence, i.e., characterizing the uncertainty of a single model, rather than an entire model class.

%\paragraph{Sequential inference with dependent data.}
%To the best of our knowledge, all known CP results rely on fixed-time concentration inequalities, which assume that the user fixes the training sample size before seeing the data. This assumption is inconsistent with the practice of control and robotics, where engineers often choose to collect additional data after ``peeking" at data that has already been collected. This motivates the use of time-uniform concentration inequalities, which are valid at arbitrary stopping times and can be used to construct valid confidence sequences when a user is sampling data in an online fashion, i.e., sequential inference \cite{howard_time-uniform_2020, howard_time-uniform_2021, waudby-smith_estimating_2022}. Most sequential inference works focus on the iid data setting \cite{howard_sequential_2022}, with the notable exception of \cite{mineiro_time-uniform_2023}, which builds confidence sequences for the time-averaged historical cumulative distribution function (see Definition~\ref{def:cdf}) of nearly arbitrary data generating processes. In Section~\ref{sec:stitched-qrc}, we use these confidence sequences, as well as the decoupling technique, to build a novel online CP algorithm that has interesting connections to smoothed online learning \cite{rakhlin_online_2011}.

\section{Problem Formulation}
Consider the time-series model
\begin{align}
    y_{t} = f_\star(x_t) + w_t   
\label{eq:model}
\end{align}
with unknown $f_\star: \mathbb{R}^{d_x} \to \mathbb{R}^{d_y}$ and noise $w_t$. Setting $y_t = x_{t+1}$ recovers an autonomous dynamical system. Because we are interested in the statistical properties of the system, rather than its dimensionality, we assume $d_x = d_y =1$ in the sequel.\footnote{Our results remain valid when $d_x, d_y > 1$, but may be conservative. Combining our results with tools from \cite{xu_conformal_2024} \cite{teneggi_how_2023} can improve efficiency. We defer details to future work.}

Suppose we have access to a learned dynamics model $\hat{f}$, which comes with no \emph{a priori} performance guarantees, and a single trajectory $Z_{1:T+k} := \{ (x_t, y_t) \}_{t=1}^{T+k}$ for some $k \in \mathbb{N}$ from the time-series model~(\ref{eq:model}), which we split into training trajectory $Z_{1:T} := \{ (x_t, y_t) \}_{t=1}^{T}$ and test point $Z_{T+k} := (x_{T+k}, y_{T+k})$. When appropriate, we assume that a test point $Z' = (x', y')$ is drawn from the stationary distribution of the process (\ref{eq:model}). 

Our goal is to use the blackbox model $\hat{f}$ (which is arbitrary but fixed, hence is elided from notation in the sequel) and the training trajectory to build a set predictor $C: \mathbb{R} \to 2^{\mathbb{R}}$ so that the event of successful coverage, $\{ y_{T+1} \in C(x_{T+1})\}$, occurs with one of two notions of high probability. CP has traditionally sought the marginal guarantee,
\begin{align}
    \underset{Z_{1:T}, Z_{T+k}}{\mathbf{P}} \Big( y_{T+k} \in C\left(x_{T+k}\right) \Big) \geq 1-\varepsilon,
\end{align}
which ensures that coverage occurs with probability $(1-\varepsilon)$ over the joint draw of the training trajectory and test point. Other works \cite{park_pac_2020, bates_distribution-free_2021} have sought the training-conditional, or probably approximately correct, guarantee, 
\begin{align}
    \underset{Z_{1:T}}{\mathbf{P}} \left( \underset{Z_{T+k}}{\mathbf{P}}\left( y_{T+k} \in C\left(x_{T+k}\right) \right) \geq 1-\varepsilon \right) \geq 1-\delta,
\label{eq:training-conditional}
\end{align}
which separately controls the randomness over the draw of the training trajectory and test point: with probability $(1-\delta)$ over the draw of the training trajectory, the test point is covered with probability $(1-\varepsilon)$. Although these guarantees are semantically different,  Proposition 2a of \cite{vovk_conditional_2012} shows that they essentially imply one another. In the sequel, we consider training-conditional guarantees, and strengthenings thereof, which follow the spirit of high probability generalization and estimation error bounds in the learning for control literature. In contrast, marginal guarantees are more aligned with results from hypothesis testing. 

Following \cite{bates_distribution-free_2021}, we view high probability guarantees of coverage as controlling the indicator loss $\mathds{1}( y \notin C(x) )$. We can also consider more general loss functions $\ell(y, C(x))$ that, like the indicator loss, decrease monotonically as prediction set size increases. 

Now we sketch specific assumptions made in each section. In Section~\ref{sec:blocked-rcps}, we require that $\{ (x_t, y_t) \}$ is asymptotically stationary and $\beta$-mixing, which means the data generating process eventually converges to a stationary distribution and samples that are sufficiently separated in time are approximately independent. The simplest system that satisfies this assumption is the strictly stable linear time-invariant system with iid Gaussian noise, which we use as a running example to build system-theoretic intuition about results. In Section~\ref{sec:decoupled-rcps}, we require that $\{ (x_t, y_t) \}$ is adapted to some filtration $\{ \mathcal{F}_t \}_{t=1}^{\infty}$ where $\mathcal{F}_t \subseteq \mathcal{F}_{t+1}$, which means that for any $t \in \mathbb{N}$, $(x_t, y_t)$ is measurable with respect to the $\sigma$-algebra $\mathcal{F}_t$. This simply prevents $(x_t, y_t)$ from depending on the future, hence is a much milder assumption than asymptotic stationarity and $\beta$-mixing. Concretely, all causal dynamical systems are adapted in this sense.

\subsection{Notation}
Given a realization of a stochastic process, $\{ x_t \}_{t=1}^{T}$, let $\mathbf{P}_i, \mathbf{P}_j$ be the marginal distributions of $x_i, x_j$. Then $\mathbf{P}_i \wedge \mathbf{P}_j$ denotes the joint distribution of $(x_i, x_j)$, while $\mathbf{P}_i \otimes \mathbf{P}_j$ denotes the product distribution of the marginals. $\mathbf{P}_i \wedge \mathbf{P}_j = \mathbf{P}_i \otimes \mathbf{P}_j$ when $x_i, x_j$ are independent. The total variation distance between two probability measures $\mathbf{P}, \mathbf{Q}$ is defined as $\| \mathbf{P} - \mathbf{Q} \|_{\text{TV}} := \sup_{A}\left|\mathbf{P}(A) - \mathbf{Q}(A) \right| \in [0, 1]$ where the supremum is taken over sets $A$ lying in an appropriate $\sigma$-algebra. 

The spectral radius of matrix $G$ is denoted by $\rho(G)$. We define the resolvent $R_G(z) = (zI - G)^{-1}$ and the $H_\infty$-norm $\| G \|_{H_\infty} = \sup_{z \in \mathbb{T}} \| G(z) \|$, where $\mathbb{T}$ denotes the complex unit circle.

\section{Blocked Risk-Controlling Prediction Sets}\label{sec:blocked-rcps}

In this section, we use the blocking technique to show that RCPS attains iid-like training-conditional guarantees (\ref{eq:training-conditional}) whenever data is generated by an asymptotically stationary and contractive dynamical system. 
%This argument can be readily generalized: other CP methods designed for iid data continue to perform well when data is generated by asymptotically stationary and contractive dynamics.%  
We begin by defining $\beta$- and $\upbeta$-mixing processes, which are statistical generalizations of asymptotically stationary and contractive systems, then state the RCPS algorithm and our results. 

\subsection{Preliminaries}
We begin by defining $\beta$- and $\upbeta$-mixing coefficients, which quantify statistical dependence in a stochastic process. 

\begin{definition}[$\beta$- and $\upbeta$-mixing processes \cite{yu_rates_1994, kuznetsov_generalization_2017}]
A stochastic process $Z_{1:\infty} = \{ Z_t \}_{t=1}^{\infty}$ has $\beta$-mixing coefficient
\begin{align*}
    \beta_{Z}(k) &:= \underset{t \in \mathbb{N}}{\sup} \Big\| \mathbf{P}_{1:t} \wedge \mathbf{P}_{t+k:\infty} - \mathbf{P}_{1:t} \otimes \mathbf{P}_{t+k:\infty} \Big\|_{\text{TV}} = \underset{t \in \mathbb{N}}{\sup} \underset{Z_{1:t}}{\mathbb{E}} \bigg[ \Big\|\mathbf{P}_{t+k:\infty}\left( \cdot | Z_{1:t}\right) - \mathbf{P}_{t+k:\infty} \Big\|_{\text{TV}} \bigg],
\end{align*}
which quantifies the degree of dependence between random variables at range $k$. 

$Z_{1:\infty}$ is said to be $\beta$-mixing if $\beta_{Z}(k) \to 0$ as $k \to \infty$, which means that sufficiently separated random variables are approximately independent. If $Z_{1:\infty}$ admits a stationary distribution $\Pi$ and 
\begin{align*}
    \upbeta_{Z}(k) := \underset{t \in \mathbb{N}}{\sup} \underset{Z_{1:t}}{\mathbb{E}} \bigg[ \Big\|\mathbf{P}_{t+k}\left( \cdot | Z_{1:t}\right) - \Pi \Big\|_{\text{TV}} \bigg] \to 0 \quad \text{as} \quad k \to \infty, 
\end{align*}
then $Z_{1:\infty}$ is said to be asymptotically stationary and $\beta$-mixing, or $\upbeta$-mixing, to $\Pi$. 
\end{definition}

When the process of interest is clear, we elide subscripts and write $\beta(k)$, $\upbeta(k)$. 

Asymptotic stationarity is a generalization of strict stationarity, which requires that $\mathbf{P}_{i:i+k} = \mathbf{P}_{i+j:i+j+k}$ $ \forall i, j, k \in \mathbb{N}.$ Strictly stationary processes have the trivial stationary distribution of $\mathbf{P}_0$, the distribution with which the process is initialized, hence are also asymptotically stationary. The converse is not necessarily true. Additionally, after conditioning on $Z_{1:t}$, $\upbeta$-mixing coefficients only consider the marginal distribution at time $t+k$, while $\beta$-mixing coefficients consider the joint distribution from time $t+k$ to $\infty$. Thus, assuming that a stationary distribution exists, controlling the $\upbeta$-mixing coefficients is an easier task than controlling the $\beta$-mixing coefficients.

This distinction is important because Oliveira et al.~\cite{oliveira_split_2022} studies the marginal coverage guarantees attained by a different conformal prediction algorithm on strictly stationary and $\beta$-mixing data. In contrast, we work with the more general assumption of asymptotically stationary and $\beta$-mixing, or $\upbeta$-mixing, data. Importantly, our assumption is more compatible with control and robotics tasks, as many dynamical systems of interest are not strictly stationary but asymptotically so, i.e., the marginal state distribution drifts over time due to noise, but eventually converges to a stationary distribution. 
The simplest such example is a strictly stable linear time-invariant (LTI) system with iid Gaussian process noise initialized from any distribution other than its stationary distribution, which we use as a running example. 

\begin{example}[Strictly stable LTI \cite{tu_least-squares_2018}]
Consider the strictly stable and autonomous LTI system  $$x_{t+1} = Ax_t + w_t$$ with spectral radius $\rho = \rho(A)<1$ and $w_t \sim N(0, I)$. The associated stochastic process $X_{1:\infty} = \{ x_t \}_{t=1}^{\infty}$ has marginal distribution $x_t \sim N(0, \Sigma_t)$, where $\Sigma_t = \sum_{j=1}^{t}(A^j)(A^j)^\top$, and stationary distribution $N(0, \Sigma_\infty)$, where $\Sigma_\infty$ is the unique solution to  the discrete-time Lyapunov equation $$A \Sigma_\mathbf{\infty} A^\top - \Sigma_\mathbf{\infty} + I = 0.$$ For any initial $\Sigma_0 \neq \Sigma_\infty$, this system is not strictly stationary, but is asymptotically so.
\label{ex:lti-1}
\end{example}

\begin{proposition}[Strictly stable LTI systems are $\upbeta$-mixing \cite{tu_least-squares_2018}]
The $\upbeta$-mixing coefficient of $X_{1:\infty}$ is upper bounded by
\begin{align*}
    \upbeta_X(k) &= \underset{t \in \mathbb{N}}{\sup} \underset{X_{1:t}}{\mathbb{E}} \bigg[ \Big\|\mathbf{P}_{t+k}\left( \cdot | x_{1:t}\right) - \Pi \Big\|_{\text{TV}} \bigg] \leq \left(\frac{\| R_{\rho^{-1}A} \|_{H_\infty}}{2} \sqrt{\mathbf{Tr}\left( \Sigma_\infty \right) + \frac{d}{1-\rho^2}} \right) \rho^k =: \Gamma \rho^k,
\end{align*}
which tends to $0$ as $k \to \infty$. Hence strictly stable LTI systems are $\upbeta$-mixing to $\Pi=N(0,\Sigma_\infty)$.
\label{prop:lti-beta}
\end{proposition}

Parsing the upper bound, the $\upbeta$-mixing coefficient is small when the underlying LTI system is stable (hence $\frac{1}{1-\rho^2}$ is small), the decay of the spectral norm of $A^k$ is well-behaved (hence the $H_\infty$-norm of the resolvent is small), and the stationary distribution has small variance (hence $\text{Tr}(\Sigma_\infty)$ is small). Using techniques from \cite{ziemann_single_2022}, Proposition~\ref{prop:lti-beta} can be generalized to contractive systems with non-iid, non-Gaussian noise at the expense of a lack of closed-form upper bounds on the mixing coefficients. This means our results apply to a larger class of realistic systems than LTI with iid Gaussian noise, with system contractivity properties and noise  ``size" determining mixing behavior. 

Because the noise terms are iid in Proposition~\ref{prop:lti-beta}, we can show that $\{ (x_t, w_{t+1}) \}$ is $\upbeta$-mixing, which lets us conclude that $\{ (x_t, x_{t+1}) \}$ is $\upbeta$-mixing as well. This is important because we consider the process $\{ (x_t, x_{t+1}) \}$ in RCPS, with $x_t$ playing the role of the feature and $x_{t+1}$ playing the role of the label at each timestep $t = 1, \dots, T-1$. Moreover, by applying Pinsker's inequality and the chain rule for the Kullback-Leibler divergence, we can see that the $\upbeta$-mixing coefficient of $\{ (x_t, x_{t+1}) \}$ is dominated by the upper bound from Proposition~\ref{prop:lti-beta}, up to a factor of two. See Section 2 of \cite{baraud_adaptive_2001} for a detailed discussion.

Now we introduce the blocking technique, which we use to analyze RCPS with $\upbeta$-mixing data. 

\begin{proposition}[Blocking technique \cite{yu_rates_1994, kuznetsov_generalization_2017, tu_least-squares_2018}] 
Draw a trajectory $Z_{1:T} = \{ Z_t \}_{t=1}^{T}$ from a $\upbeta$-mixing process with stationary distribution $\Pi$. Fix $m, n \in \mathbb{N}$ and suppose without loss of generality that $T = mn$. Construct $n$ blocks, or subsamples, of size $m$ each, where
\begin{align*}
    Z_{(j)} = \{ Z_t: (t-1 \mod n) = k-1 \} \quad \text{for} \quad j = 1, 2, \dots, a.
\end{align*}
Let $\tilde{Z}_\Pi$ be a block of $m$ iid draws from $\Pi$. Then for any measurable function $h$ that takes values in $[-B_1, B_2] \subset \mathbb{R}$,
\begin{align*}
    \left| \mathbb{E}\Big[ h \big( \tilde{Z}_\Pi \big) \Big] - \mathbb{E}\left[ h \big( Z_{(j)} \big) \right] \right| \leq (B_1 + B_2)m \upbeta_Z(n). 
\end{align*}
\end{proposition}

Intuitively, the blocking technique transforms the analysis of $Z_{(j)}$, a subsample of non-iid random variables, into one of $\tilde{Z}_\Pi$, a subsample of iid random variables, at the expense of an additive error term that can be minimized with our choice of $m, n$. Importantly, this blocking technique is only used to analyze the performance of an algorithm, in contrast to blocking techniques from related works that must be algorithmically implemented \cite{chernozhukov_exact_2018, bhatnagar_improved_2023}. 

\subsection{Results}
Now we state the RCPS algorithm and present results. For data generated by (\ref{eq:model}) and set predictor $C: \mathbb{R} \to 2^{\mathbb{R}}$, let $\ell$ be a loss function that generalizes the indicator loss $\mathds{1} \left( y \notin C(x) \right)$ associated with coverage, i.e., $\ell$ takes values in $[0, B] \subset \mathbb{R}_+$ and decreases monotonically as prediction set size increases. Further suppose the set predictor $C$ is parameterized by $\lambda \in \mathbb{R}_+$, which controls the radius of the prediction set, i.e., $\lambda < \lambda' \implies C_\lambda(x) \subset C_{\lambda'}(x)$. 

Given risk tolerance $\varepsilon \in [0, B]$, failure probability $\delta \in (0, 1)$, and iid data $Z_{1:n} \cup Z$, the standard version of RCPS \cite{bates_distribution-free_2021} selects 
\begin{align*}
    \hat{\lambda} = \inf\left\{ \lambda: \frac{1}{n}\sum_{j=1}^{n} \ell(y_j, C_\lambda(x_j)) \leq \varepsilon - \sqrt{\frac{\log(1/\delta)}{n}} \right\},
\end{align*}
which can be easily computed using the empirical risk $\frac{1}{n}\sum_{j=1}^{n} \ell \left( y_j, C_{\hat{\lambda}}(x_j) \right)$. Under the iid data assumption, the analysis of standard RCPS boils down to controlling the empirical process
\begin{align*}
    \Phi(Z_{1:n}) := \underset{Z}{\mathbb{E}}\left[ \ell \left(y, C_{\hat{\lambda}}(x) \right) \right] - \frac{1}{n}\sum_{j=1}^{n} \ell \left( y_j, C_{\hat{\lambda}}(x_j) \right).
\end{align*}
Applying Hoeffding's inequality yields
\begin{align*}
    &\underset{Z_{1:n}}{\mathbf{P}}\left( \underset{Z}{\mathbb{E}}\left[ \ell \left(y, C_{\hat{\lambda}}(X) \right) \right] \leq \frac{1}{n}\sum_{j=1}^{n} \ell \left( y_j, C_{\hat{\lambda}}(x_j) \right) + \sqrt{\frac{\log(1/\delta)}{n}} \right) = \underset{Z_{1:n}}{\mathbf{P}}\left( \underset{Z}{\mathbb{E}}\left[ \ell \left(y, C_{\hat{\lambda}}(X) \right) \right] \leq \varepsilon \right) \leq 1-\delta.
\end{align*}

In contrast, given risk tolerance $\varepsilon \in [0, B]$, failure probability $\delta \in (0, 1)$, and $\upbeta$-mixing data $Z_{1:T}$, our version of RCPS selects
\begin{align}
    \hat{\lambda} = \inf\left\{ \lambda: \frac{1}{n}\sum_{t=1}^{n} \ell(y_t, C_\lambda(x_t)) \leq \varepsilon \right\},
\label{eq:blocked-rcps}
\end{align}
i.e., we treat the empirical risk as if it were the true risk, with no Hoeffding-inspired adjustment, when selecting the prediction set radius parameter $\hat \lambda$. 

We first study risk control with respect to the stationary distribution of the system, $\Pi$, then with respect to the marginal distribution $\mathbf{P}_{T+k}$ of the test point $Z_{T+k} = (x_{T+k}, y_{T+k})$. 

\begin{theorem}[Blocked RCPS, $\Pi$.] 
Fix risk tolerance $\varepsilon \in [0, B]$ and failure probability $\delta \in (0, 1)$. Suppose data is generated by a $\upbeta$-mixing process with $\upbeta(k) = O\left( k^{-1} \right)$.\footnote{This is a mild assumption. As seen in Proposition~\ref{prop:lti-beta}, the mixing coefficients of systems of interest often decay exponentially in $k$, readily satisfying the polynomial decay requirement.} Fix block length $m$ and number of blocks $n$, and assume $T=mn$ without loss of generality. Further assume that $m$ and $n$ are such that $\delta > BT\upbeta(n)$, which is always possible for sufficiently large $T = O\left( \text{poly} \left( \delta^{-1} \right) \right)$.\footnote{The assumption $\delta > BT\upbeta(n)$ is not restrictive when $\upbeta(k) = O \left( k^{-1} \right)$. Assume the worst case and let $\upbeta(k) = k^{-(1+\xi)}$ for some arbitrarily small constant $\xi>0$. Further assuming that $n$ is a constant proportion of $T$, it only takes a training trajectory of length $T = O\left( \delta^{-1/\xi} \right) = \text{poly} \left( \delta^{-1} \right)$ to ensure that $\delta > BT\upbeta(n)$ holds. As seen in Example~\ref{ex:lti-2}, the required burn-in time is often much milder than this worst-case analysis suggests.} Then for trajectory $Z_{1:T} = \{ (x_t, y_t) \}_{t=1}^{T}$ and $\hat{\lambda}$ set according to (\ref{eq:blocked-rcps}), we have that
    \begin{align*}
    \underset{Z_{1:T}}{\mathbf{P}} \left( \underset{\Pi}{\mathbb{E}}\left[ \ell(y', C_{\hat{\lambda}}(x')) \right] \leq \varepsilon + \gamma \right) \geq 1-\delta.
\end{align*}
The expectation is taken with respect to a draw of $Z' = (x', y')$ from the stationary distribution $\Pi$, and the additive discrepancy term
\begin{align*}
    \gamma = \sqrt{\frac{\log\left( \frac{n}{\delta - BT\upbeta(n)} \right)}{m}}.
\end{align*}
\label{thm:blocked-rcps}
\end{theorem}

Theorem~\ref{thm:blocked-rcps} states that for a sufficiently long $\upbeta$-mixing training trajectory, the upper bound on true risk $\underset{\Pi}{\mathbb{E}}\left[ \ell(y', C_{\hat{\lambda}}(x')) \right]$ can be made arbitrarily close to $\varepsilon$ at the mixing-deflated rate $\tilde{O}\left( m^{-1/2} \right)$, where $\tilde{O}\left( \cdot \right)$ elides logarithmic factors.  We revisit this result in Section~\ref{sec:decoupled-rcps}, and leverage the more general decoupling technique to avoid this deflation and match the $T^{-1/2}$ rate from the iid setting. This suggests RCPS with $\upbeta$-mixing data is even easier than is suggested by the blocking technique. 

Moreover, because this result does not rely on any particular values of $m, n$, they can be selected (as part of the analysis of RCPS, not its implementation) to balance the required burn-in time for $T$ and the $\gamma$ discrepancy term. In Example~\ref{ex:lti-2}, we sketch a concrete strategy for selecting the number and size of blocks when a closed-form upper bound on the mixing coefficient is known.

\begin{proof}[Proof of Theorem 1]
We adapt the proof of Theorem 2.2 in \cite{tu_least-squares_2018}. 

For $j \in \{ 1, 2, \dots, n \}$, let $I_j$ be the set of indices included in the $j$th block. Then
\begin{align*}
    \underset{Z_{1:T}}{\mathbf{P}}&\left( \underset{\Pi}{\mathbb{E}}\left[ \ell \left( y', C_{\hat{\lambda}}(x') \right) \right] - \frac{1}{T}\sum_{t=1}^{T} \ell \left(y_t, C_{\hat{\lambda}}(x_t)\right) > \gamma \right) \\
    &= \underset{Z_{1:T}}{\mathbf{P}}\left( \frac{1}{n} \sum_{j=1}^{n} \frac{1}{m} \sum_{t \in I_j} \left(\underset{\Pi}{\mathbb{E}}\left[ \ell \left( y', C_{\hat{\lambda}}(x') \right) \right] - \ell \left(y_t, C_{\hat{\lambda}}(x_t)\right) \right) > \gamma \right) \\
    &\leq \sum_{j=1}^{n} \underset{Z_{(j)}}{\mathbf{P}}\left( \frac{1}{m} \sum_{t \in I_j} \left(\underset{\Pi}{\mathbb{E}}\left[ \ell \left( y', C_{\hat{\lambda}}(x') \right) \right] - \ell \left(y_t, C_{\hat{\lambda}}(x_t)\right) \right) > \gamma \right) \\
    &\leq BT\upbeta(n) + n\underset{\tilde{Z}_{\Pi}}{\mathbf{P}}\left( \underset{\Pi}{\mathbb{E}}\left[ \ell \left( y', C_{\hat{\lambda}}(x') \right) \right] - \frac{1}{m} \sum_{t \in I_j} \left(\ell \left(y_t, C_{\hat{\lambda}}(x_t)\right) \right) > \gamma \right) \\
    &\leq BT\upbeta(n) + \left( \delta - BT\upbeta(n) \right) = \delta.
\end{align*}
The first inequality holds by the union bound: if
\begin{align*}
    \frac{1}{n} \sum_{j=1}^{n} \frac{1}{m} \sum_{t \in I_j} \left(\underset{\Pi}{\mathbb{E}}\left[ \ell \left( y', C_{\hat{\lambda}}(x') \right) \right] - \ell \left(y_t, C_{\hat{\lambda}}(x_t)\right) \right) > \gamma
\end{align*}
holds, then there must exist an index $j$ for which 
\begin{align*}
    \frac{1}{m} \sum_{t \in I_j} \left(\underset{\Pi}{\mathbb{E}}\left[ \ell \left( y', C_{\hat{\lambda}}(x') \right) \right] - \ell \left(y_t, C_{\hat{\lambda}}(x_t)\right) \right) > \gamma
\end{align*} 
holds. The second inequality follows from Proposition 2, and the third inequality follows from setting $\gamma = \sqrt{\frac{\log\left( \frac{n}{\delta - BT\upbeta(n)} \right)}{m}}$ and applying a one-sided Hoeffding's inequality. The conclusion then follows by setting $\hat \lambda$ according to~\eqref{eq:blocked-rcps}. 
\end{proof}

Importantly, the $\gamma$ discrepancy term can be made arbitrarily small for sufficiently large $m, n$ (and hence, $T$), but suffers the deflated rate $\tilde{O}\left( m^{-1/2} \right)$. An improved rate that adapts to the variance of $\Pi$ can be shown using Bernstein's inequality, or other specialized concentration inequalities. Importantly, unlike in standard RCPS, the choice of concentration inequality only affects the analysis, not the implementation of the algorithm.

Next, we return to our running LTI example and revisit our choice of $m, n$. 

\begin{example}[Strictly stable LTI \cite{tu_least-squares_2018}]
When data is drawn from a strictly stable LTI system with iid Gaussian noise, we can leverage Proposition~\ref{prop:lti-beta} and select
\begin{align*}
    n = \left\lceil \frac{1}{1-\rho} \log \left( \frac{4 \Gamma BT}{\delta} \right) \right\rceil \quad \text{and} \quad \gamma = \sqrt{\frac{\log\left( \frac{2n}{\delta} \right)}{m}},
\end{align*} 
to enforce
\begin{align*}
    BT\upbeta(n) + n\underset{\tilde{Z}_{\Pi}}{\mathbf{P}}\left( \underset{\Pi}{\mathbb{E}}\left[ \ell \left( y', C_{\hat{\lambda}}(x') \right) \right] - \frac{1}{m} \sum_{t \in I_j} \left(\ell \left(y_t, C_{\hat{\lambda}}(x_t)\right) \right) > \gamma \right) \leq 2BT \Gamma \rho^n + \frac{\delta}{2} = \delta.
\end{align*}
\label{ex:lti-2}
\end{example}

Since $T \geq n = O\left( \log \left( \delta^{-1} \right) \right)$, Example~\ref{ex:lti-2} shows that the required burn-in of $T$ can be only logarithmic in $\delta^{-1}$, which is far better than the worst-case polynomial dependence discussed previously. This improved dependence holds generally for contractive systems, whose mixing coefficients decay exponentially in range $k$. However, because $T=ma$ by assumption, selecting $n$ using the mixing coefficient upper bound can lead to a small value of $m$ and consequently, a large value of $\gamma$. In practice, block number and size should be selected to balance the required burn-in time for $T$ and the $\gamma$ discrepancy term.

Now we consider risk control with respect to the marginal distribution $\mathbf{P}_{T+k}$ of the test point $Z_{T+k} = (x_{T+k}, y_{T+k})$. 

\begin{corollary}[Blocked RCPS, $\mathbf{P}_{T+k} $]
Fix risk tolerance $\varepsilon \in [0, B]$ and failure probability $\delta \in (0, 1)$. Under the assumptions of Theorem~\ref{thm:blocked-rcps}, we have that
\begin{align*}
    \underset{Z_{1:T}}{\mathbf{P}} \left( \underset{\mathbf{P}_{T+k}}{\mathbb{E}}\left[ \ell(y_{T+k}, C_{\hat{\lambda}}(x_{T+k})) \right] \leq \varepsilon + \gamma + \nu \right) \geq 1-\delta,
\end{align*}
where the expectation is taken over the marginal distribution $\mathbf{P}_{T+k}$ of test point $(x_{T+k}, y_{T+k})$, $\gamma = \tilde{O}\left( m^{-1/2} \right)$ as in Theorem~\ref{thm:blocked-rcps}, and $\nu$ is an additional discrepancy term that is at most $B\upbeta(k)$.
\label{cor:blocked-rcps}
\end{corollary}

\begin{proof}[Proof of Corollary 1]
From Theorem 1, we have that 
\begin{align*}
    1-\delta &\leq \underset{Z_{1:T}}{\mathbf{P}} \left( \underset{\Pi}{\mathbb{E}}\left[ \ell(y', C_{\hat{\lambda}}(x')) \right] \leq \varepsilon + \gamma \right) \\
    &= \underset{Z_{1:T}}{\mathbf{P}} \left( \underset{\mathbf{P}_{T+k}}{\mathbb{E}}\left[ \ell(y_{T+k}, C_{\hat{\lambda}}(x_{T+k})) \right] \leq \varepsilon + \gamma + \underset{\mathbf{P}_{T+k}}{\mathbb{E}}\left[ \ell(y_{T+k}, C_{\hat{\lambda}}(x_{T+k})) \right] - \underset{\Pi}{\mathbb{E}}\left[ \ell(y', C_{\hat{\lambda}}(x')) \right] \right). 
\end{align*}
Then we can rewrite
\begin{align*}
    \nu &= \underset{\mathbf{P}_{T+k}}{\mathbb{E}}\left[ \ell(y_{T+k}, C_{\hat{\lambda}}(x_{T+k})) \right] - \underset{\Pi}{\mathbb{E}}\left[ \ell(y', C_{\hat{\lambda}}(x')) \right] \leq B \| \mathbf{P}_{T+k} - \Pi \|_{\text{TV}}.
\end{align*}
By the law of iterated expectations and Jensen's inequality, we can further upper bound 
\begin{align*}
    \| \mathbf{P}_{T+k} - \Pi \|_{\text{TV}} \leq \underset{Z_{1:T}}{\mathbb{E}} \left[ \left\| \mathbf{P}_{T+k}\left( \cdot | Z_{1:T} \right) - \Pi \right\|_{\text{TV}} \right] \leq \upbeta(k).
\end{align*}
\end{proof}

It is easier to perform RCPS on a test point with $k > 1$ than a test point with $k = 1$, since the definition of $\upbeta$-mixing implies the former's marginal distribution is closer to the stationary distribution $\Pi$ in total variation distance than the latter's.

\section{Decoupled Risk-Controlling Prediction Sets}
\label{sec:decoupled-rcps}

The blocked RCPS approach has two key shortcomings. First, there is no straightforward way to relax the mixing assumption, which means blocked RCPS cannot be used for, e.g., marginally stable LTI systems, as a linear system mixes if and only if its spectral radius is strictly less than one. More generally, blocked RCPS gives no guarantees for noncontractive systems. Second, in the setting of Corollary~\ref{cor:blocked-rcps}, where we consider risk control with respect to the marginal distribution of the test point $(x_{T+k}, y_{T+k})$, the blocking analysis does not suggest straightforward ways to modify the RCPS algorithm to attain control over the additional $\nu \leq B\upbeta(k)$ error term. 

In this section, we introduce the decoupling technique \cite{pena_decoupling_1999}, which simultaneously relaxes the mixing assumption and hints at how the RCPS algorithm can be modified to perform well in more general non-iid data settings.

\subsection{Preliminaries}
We first introduce adapted processes, which are stochastic processes that satisfy a natural notion of causality.  Formally, this notion is encoded in a nondecreasing sequence of $\sigma$-algebras, which means the definition of adapted processes reduces to a measurability condition. All causal dynamical systems are adapted in this sense. 

\begin{definition}[Adapted process \cite{pena_decoupling_1999}]
    Let $( \Omega, \mathcal{F}, \mathbf{P}; \{ \mathcal{F}_t \}_{t=1}^{\infty})$ be a filtered probability space with a sequence of $\sigma$-algebras, or filtration, $\{ \mathcal{F}_t \}_{t = 1}^{\infty}$ that satisfies $\mathcal{F}_t \subseteq \mathcal{F}_{t+1}$ and $\mathcal{F}_t \subset \mathcal{F}$ for all $t \in \mathbb{N}$. A process $\{ Z_t \}_{t=1}^{\infty}$ is said to be \textit{adapted} to filtration $\{ \mathcal{F}_t \}_{t=1}^{\infty}$ if all $Z_t$ are $\mathcal{F}_t$-measurable.
\label{def:adapted}
\end{definition}

We simply say a sequence is adapted if there exists some filtration to which it is adapted, but the filtration itself is of no interest. A sequence being adapted is a sufficient condition for there to exist a decoupled tangent sequence of conditionally independent random variables \cite{pena_decoupling_1999}. This decoupled sequence, in turn, can be used to evaluate the mean (or other distributional quantities that we do not consider, e.g., higher moments) of functions of the original sequence. The decoupled sequence is the philosophical analog of the iid samples from the stationary distribution used in the blocking technique, but is more general, in the sense that no blocking scheme is necessary and the resulting random variables are only conditionally independent, not outright iid. 

\begin{proposition}[Decoupling technique \cite{pena_decoupling_1999}] \label{prop:decoup}
    Let $\{ Z_t \}_{t=1}^{\infty}$ be a stochastic process adapted to some nondecreasing filtration $\{ \mathcal{F}_t \}_{t=1}^{\infty}$ that is contained in $\sigma$-algebra $\mathcal{F}$. Then there always exists a \textit{decoupled tangent sequence} $\{ Z_t' \}_{t=1}^{\infty}$ that satisfies the following conditions:
    \begin{enumerate}
        \item $\mathbf{P}\left( Z_t | \mathcal{F}_{t-1} \right) = \mathbf{P}\left( Z_t' | \mathcal{F}_{t-1} \right)$ for all $t \in \mathbb{N}$, where $\mathbf{P}( \cdot | \mathcal{F}_t)$ denotes the conditional probability measure given $\mathcal{F}_t$, and
        \item $\{ Z_t' \}_{t=1}^{\infty}$ is conditionally independent given some $\sigma$-algebra $\mathcal{G} \subset \mathcal{F}$ for which $\mathbf{P}(Z_t' | \mathcal{F}_{t-1}) = \mathbf{P}(Z_t' | \mathcal{G})$. Often, $\mathcal{G}$ can be taken as the $\sigma$-algebra induced by $\{ Z_t \}_{t=1}^{\infty}$, i.e., $\sigma\left( \{ Z_t \} \right)$.
    \end{enumerate}
Then by the linearity of expectation and the law of iterated expectations, it follows that
\begin{align*}
    \mathbb{E} \left[ \sum_{t=1}^{T} Z_t \right] &= \sum_{t=1}^{T} \mathbb{E} \left[ \mathbb{E} \left[ Z_t | \mathcal{F}_{t-1}\right] \right] = \sum_{t=1}^{T} \mathbb{E} \left[ \mathbb{E} \left[ Z_t' | \mathcal{F}_{t-1}\right] \right] = \sum_{t=1}^{T} \mathbb{E} \left[ \mathbb{E} \left[ Z_t' | \mathcal{G} \right] \right] = \mathbb{E} \left[ \sum_{t=1}^{T} Z_t' \right].
\end{align*}
\end{proposition}

\subsection{Results}
We first generalize RCPS and study the weighted empirical risk $\sum_{t=1}^{T} w_t \cdot \ell(y_t, C_\lambda(x_t))$ where $w_t \geq 0$ for all $t$ and $\sum_{t=1}^{T} w_t = 1$. We recover the version of RCPS studied in Section~\ref{sec:blocked-rcps} with uniform weights $w_t = T^{-1}$. Then letting $w := (w_1, \dots, w_{T})$, we define
\begin{align}
    \hat{\lambda}(w) := \inf\left\{ \lambda: \sum_{t=1}^{T} w_t \cdot \ell(y_t, C_\lambda(x_t)) \leq \varepsilon \right\}.
\label{eq:decoupled-rcps}
\end{align}
\begin{theorem}[Decoupled RCPS]
    Fix risk tolerance $\varepsilon \in [0, B]$ and failure probability $\delta \in (0, 1)$. Suppose the trajectory $Z_{1:T+1} = \{ (x_t, y_t) \}_{t=1}^{T+1}$ is adapted to some filtration $\{ \mathcal{F}_t \}_{t=1}^{T+1}$. Then for any choice of weights $w$, setting $\hat{\lambda}(w)$ according to \eqref{eq:decoupled-rcps} attains
    \begin{align*}
        \underset{Z_{1:T}}{\mathbf{P}} \left( \underset{Z_{T+1}}{\mathbb{E}}\left[ \ell(y_{T+1}, C_{\hat{\lambda}}(x_{T+1})) | Z_{1:T} \right] \leq \varepsilon + \gamma(w) + \eta(w) \right) \geq 1-\delta,
    \end{align*}
    where the expectation is taken with respect to the conditional distribution of the test point given the realized training trajectory, $\gamma(w) := B\| w \|_2 \sqrt{8\log(1/\delta)}$, and
    \begin{align*}
        \eta(w) :=  \left( \mathbb{E}\left[ \ell \left(y_{T+1}, C_{\hat{\lambda}}(x_{T+1}) \right) | Z_{1:T} \right] - \sum_{t=1}^{T} w_t \cdot \mathbb{E} \left[ \ell \left( y_t, C_{\hat{\lambda}}(x_t) \right) | Z_{1:t-1} \right] \right).
    \end{align*}
    \label{thm:decoupled-rcps}
\end{theorem}

First, we note that the Theorem~\ref{thm:decoupled-rcps} is stronger than Theorem~\ref{thm:blocked-rcps} and Corollary~\ref{cor:blocked-rcps} in the sense that the former controls the conditional risk given the specific training trajectory, while the latter two control the marginal risk, which averages over all possible training trajectories.\footnote{When data is iid, $\mathbb{E}\left[ \ell(x_{T+1}, C_{\hat{\lambda}}(x_{T+1})) \right]$ and $\mathbb{E}\left[ \ell(x_{T+1}, C_{\hat{\lambda}}(x_{T+1})) | Z_{1:T} \right]$ are identical because independence means conditioning on training data does not change the marginal expectation. But when data is dependent, the marginal and conditional expectations are not generally identical.}

Next, we highlight the dependence of the error terms $\gamma(w)$ and $\eta(w)$ on the choice of weights $w$. The first term, $B \| w \|_2 \sqrt{8\log(1/\delta)}$, is analogous to the $\gamma$ error term from Theorem~\ref{thm:blocked-rcps}. Importantly, for the choice of uniform weights $w_t = T^{-1}$, this first term scales on the order of $T^{-1/2}$, which improves on the rate from Theorem 1.\footnote{This improved rate suggests that whenever data is generated by asymptotically stationary and contractive dynamics, as in Section~\ref{thm:blocked-rcps}, risk control with $\gamma = 0$ can be attained if we set $$\hat{\lambda} = \inf\left\{ \lambda: \frac{1}{T}\sum_{t=1}^{n} \ell(y_t, C_\lambda(x_t)) \leq \varepsilon - O \left( \sqrt{\frac{\log(1/\delta)}{T}} \right) \right\},$$ as in standard RCPS for iid data. This affirms our conclusion that risk control with $\upbeta$-mixing data, and with respect to the stationary distribution, is no harder than risk control with iid data. Even risk control with respect to the marginal distribution of the test point is easier than suggested by Corollary~\ref{cor:blocked-rcps}, since the $\gamma$ error term can be offset.} The second term, $\eta(w)$, is analogous to the $\nu \leq B\upbeta(k)$ term from Corollary~\ref{cor:blocked-rcps}. If we assume $\upbeta$-mixing data, this term can be exactly upper bounded by the weighted sum of the appropriate $\upbeta$-mixing coefficients. In the more general setting with adapted data, $\eta(w)$ serves as a measure of nonstationarity and noncontractivity: it captures the difference between the future, $\mathbb{E}\left[ \ell \left(y_{T+1}, C_{\hat{\lambda}}(x_{T+1}) \right) | Z_{1:T} \right]$, and the best guess of the future given the past, $\sum_{t=1}^{T} w_t \cdot \mathbb{E} \left[ \ell \left( y_t, C_{\hat{\lambda}}(x_t) \right) | Z_{1:t-1} \right]$. This difference naturally grows with, e.g., explosive systems.\footnote{These results are similar in spirit with those of \cite{barber_conformal_2023}, which works with a different CP algorithm and marginal guarantees, but similarly upper bounds a dependence-driven discrepancy term by a weighted sum of total variation distances. But our definition of $\eta(w)$ is tighter and more intuitive for temporally correlated processes of interest.}

Importantly, a user can aim to select weights $w$ to jointly optimize the sum $\gamma(w)+\eta(w)$. This is an important degree of freedom afforded by the decoupling technique: in Theorem~\ref{thm:blocked-rcps} and Corollary~\ref{cor:blocked-rcps}, the user could control the $\gamma$ error term, but not the additional mixing coefficient-valued error term $\nu$. With decoupled RCPS, we see that the user can always enjoy the optimal $T^{-1/2}$ rate in the $\gamma$ term, but can also choose to tolerate a slower rate if a non-uniform choice of weights better balances the $\gamma(w) + \eta(w)$ sum in the finite sample. 

There are at least two approaches to selecting weights. First, using Theorem 11 from \cite{kuznetsov_discrepancy-based_2020}, the user can estimate $\eta(w)$ from data for multiple values of $w$, which requires assumptions to limit the degree of nonstationarity in the underlying dynamical system, and potentially, access to multiple trajectories of data. Alternatively, we conjecture that an online estimation approach is also possible. In either case, the decoupling technique clearly delineates a design space for adaptive RCPS algorithms---involving weighted empirical risk minimization---that can perform well in general non-iid data settings.

\begin{proof}[Proof of Theorem 2]
We adapt the proof of Theorem 1 in \cite{kuznetsov_discrepancy-based_2020}. 

For notational convenience, let $\ell(Z_t) := \ell \left( y_t, C_{\hat{\lambda}}(x_t) \right).$ Fix arbitrary weights $w$ and elide it when convenient, e.g., $\gamma = \gamma(w)$. Then it suffices to show that 
\begin{align*}
    \underset{Z_{1:T}}{\mathbf{P}}& \left( \underset{Z_{T+1}}{\mathbb{E}} \left[ \ell(Z_{T+1}) | Z_{1:T} \right] - \sum_{t=1}^{T} w_t \ell(Z_t) > \gamma(w) + \eta(w) \right) \\
    &= \underset{Z_{1:T}}{\mathbf{P}} \left( \sum_{t=1}^{T} w_t \left( \underset{Z_t}{\mathbb{E}} \left[ \ell \left( Z_t \right) | Z_{1:t-1} \right] - \ell \left( Z_t \right) \right) > \gamma \right) \leq \delta. 
\end{align*}
For any $a>0$, we have by the Chernoff bound that
\begin{multline*}
    \underset{Z_{1:T}}{\mathbf{P}} \left( \sum_{t=1}^{T} w_t \left( \underset{Z_t}{\mathbb{E}} \left[ \ell \left( Z_t \right) | Z_{1:t-1} \right] - \ell \left( Z_t \right) \right) > \gamma \right) \\  \leq \exp \left( -a \gamma \right) \underset{Z_{1:T}}{\mathbb{E}} \left[ \exp \left( a \sum_{t=1}^{T} w_t \left( \underset{Z_t}{\mathbb{E}} \left[ \ell(Z_t) | Z_{1:t-1} \right] - \ell(Z_t) \right) \right) \right].
\end{multline*}
Also, 
\begin{align*}
    \underset{Z_{1:T}}{\mathbb{E}}& \left[ \exp \left( a \sum_{t=1}^{T} w_t \left( \underset{Z_t}{\mathbb{E}} \left[ \ell(Z_t) | Z_{1:t-1} \right] - \ell(Z_t) \right) \right)\right] \\ 
    &= \underset{Z_{1:T}}{\mathbb{E}} \left[ \exp \left( a \sum_{t=1}^{T} w_t \left( \underset{Z_t'}{\mathbb{E}} \left[ \ell \left( Z_t' \right) | Z_{1:T} \right] - \ell(Z_t) \right) \right) \right] \\
    &\leq \underset{Z_{1:T}}{\mathbb{E}} \left[ \underset{Z_{1:T}'}{\mathbb{E}} \left[ \exp \left( a \sum_{t=1}^{T} w_t \left( \ell \left( Z_t' \right) - \ell(Z_t) \right) \right) \bigg| Z_{1:T} \right] \right] \\
    &= \underset{Z_{1:T}, Z_{1:T}'}{\mathbb{E}} \left[ \exp \left( a \sum_{t=1}^{T} w_t \left( \ell \left( Z_t' \right) - \ell(Z_t) \right) \right) \right],
\end{align*}
where the first equality holds by Proposition~\ref{prop:decoup}, which implies
\begin{align*}
    \underset{Z_t}{\mathbb{E}} \left[ \ell(Z_t) | Z_{1:t-1} \right] = \underset{Z_t'}{\mathbb{E}} \left[ \ell \left( Z_t' \right) | Z_{1:t-1} \right] = \underset{Z_t'}{\mathbb{E}} \left[ \ell \left( Z_t' \right) | Z_{1:T} \right];
\end{align*} 
the inequality holds by Jensen's inequality; and the second equality holds by the law of iterated expectations. 

We then reduce the expectation over $Z_{1:T}, Z_{1:T}'$ to an expectation over better behaved random variables via a sequential symmetrization argument, following \cite{kuznetsov_discrepancy-based_2020, rakhlin_online_2011}. The key idea is introducing auxiliary random variables $\sigma = (\sigma_1, \sigma_2, \dots, \sigma_T)$, where $\sigma_t$ is distributed uniformly over $\{-1, 1\}$ and drawn independently of $Z_t, Z_t'$, so that the expression inside the expectation becomes
\begin{align*}
    \exp \left( a \sum_{t=1}^{T} \sigma_t w_t \left( \ell(Z_t') - \ell(Z_t) \right) \right).
\end{align*}

Consider $t = 1$. If $\sigma_1 = 1$, the expression $w_1\left( \ell(Z_1') - \ell(Z_1) \right)$ is unchanged. But if $\sigma_1 = -1$, the order of subtraction is flipped, i.e., $w_1\left( \ell(Z_1) - \ell(Z_1') \right)$. This means $Z_1$ becomes part of the tangent sequence and $Z_1'$, the original sequence. Now consider $t=2$. The draw of $Z_2$ depends on whether $Z_1$ or $Z_1'$ is in the original sequence, and the draw of $Z_2'$ similarly depends on whether $Z_1$ or $Z_1'$ is part of the tangent sequence. This argument can be iteratively applied to all indices $t \in \{ 1, 2, \dots, T\}$, which leads to a binary tree-like dependence structure. 

Define a binary tree $\mathbf{z} = (\mathbf{z}_1, \dots, \mathbf{z}_T)$ as a $T$-tuple of mappings $\mathbf{z}_t: \{ -1, 1\}^{t-1} \to \mathbb{R} \times \mathbb{R}$, i.e., from $\sigma_1, \dots, \sigma_{t-1}$ to a value of $Z_t = (x_t, y_t)$. Similarly, define the tangent binary tree $\mathbf{z}' = (\mathbf{z}_1', \dots, \mathbf{z}_T')$ as the $T$-tuple of maps $\mathbf{z}_t'$ from $\sigma_1, \dots, \sigma_{t-1}$ to $Z_t'=(x_t', y_t')$. Intuitively, these trees track ``who is tangent to who."

\begin{proposition}[Sequential symmetrization \cite{kuznetsov_discrepancy-based_2020, rakhlin_online_2011}]
\begin{align*}
    \underset{Z_{1:T}, Z_{1:T}'}{\mathbb{E}}& \left[ \exp \left( a \sum_{t=1}^{T} w_t \left( \ell \left( Z_t' \right) - \ell(Z_t) \right) \right) \right] = \underset{\mathbf{z}', \mathbf{z}}{\mathbb{E}} \left[  \underset{\sigma}{\mathbb{E}} \left[ \exp \left( a \sum_{t=1}^{T} \sigma_t w_t \left( \ell(\mathbf{z}_t') - \ell(\mathbf{z}_t) \right) \right) \right] \right].
\end{align*}
\label{prop:symmetric}
\end{proposition}
Intuitively, Proposition~\ref{prop:symmetric} states that the expectation over the joint draw of $Z_{1:T}, Z_{1:T}'$ is equivalent to the expectation over the joint draw of trees $\mathbf{z}, \mathbf{z}'$ and a path $\sigma$ through those trees. A proof follows from Lemma 2 of \cite{kuznetsov_discrepancy-based_2020} and Theorem 3 of \cite{rakhlin_online_2011}. Proceeding to use the result, we have that
\begin{align*}
    \underset{\mathbf{z}', \mathbf{z}}{\mathbb{E}} &\left[  \underset{\sigma}{\mathbb{E}} \left[ \exp \left( a \sum_{t=1}^{T} \sigma_t w_t \left( \ell(\mathbf{z}_t') - \ell(\mathbf{z}_t) \right) \right) \right] \right] \\
    &= \underset{\mathbf{z}', \mathbf{z}}{\mathbb{E}} \left[  \underset{\sigma}{\mathbb{E}} \left[ \exp \left( a \sum_{t=1}^{T} \sigma_t w_t  \ell(\mathbf{z}_t')  + \sum_{t=1}^{T} -\sigma_t w_t \ell(\mathbf{z}_t) \right) \right] \right] \\
    &\leq \frac{1}{2} \underset{\mathbf{z}'}{\mathbb{E}} \left[ \underset{\sigma}{\mathbb{E}}\left[ \exp \left( 2a \sum_{t=1}^{T} \sigma_t w_t \ell(\mathbf{z}_t') \right) \right] \right] + \frac{1}{2} \underset{\mathbf{z}}{\mathbb{E}} \left[ \underset{\sigma}{\mathbb{E}}\left[ \exp \left( 2a \sum_{t=1}^{T} \sigma_t w_t \ell(\mathbf{z}_t) \right) \right] \right] \\
    & = \underset{\mathbf{z}}{\mathbb{E}} \left[ \underset{\sigma}{\mathbb{E}}\left[ \exp \left( 2a \sum_{t=1}^{T} \sigma_t w_t \ell(\mathbf{z}_t) \right) \right] \right],
\end{align*}
where the inequality holds by Young's inequality and the second equality holds by symmetry. 

We then apply Hoeffding's lemma for bounding moment generating functions, which states that 
\begin{align*}
    \underset{\sigma_T, \mathbf{z}_T}{\mathbb{E}}& \Big[ \exp \left( (2aw_t) \sigma_T \ell(\mathbf{z}_T) \right) \Big] \leq \exp \left( 2aw_t \underset{\sigma_T, \mathbf{z}_T}{\mathbb{E}}\left[ \sigma_T \ell(\mathbf{z}_T) \right] + \frac{(2aw_t)^2 (2B)^2}{8}\right) = \exp \left( 2a^2 w_t^2 B^2 \right).
\end{align*}
The equality holds because $\sigma_T, \mathbf{z}_T$ are independent and
\begin{align*}
    \underset{\sigma_T, \mathbf{z}_T}{\mathbb{E}}\left[ \sigma_T \ell(\mathbf{z}_T) \right] = \underset{\mathbf{z}_T}{\mathbb{E}}\left[ \ell(\mathbf{z}_T) \right] - \underset{\mathbf{z}_T}{\mathbb{E}}\left[ \ell(\mathbf{z}_T) \right] = 0 .  
\end{align*}
Then we conclude that
\begin{align*}
    \underset{\mathbf{z}, \sigma}{\mathbb{E}} \left[ \exp \left( 2a \sum_{t=1}^{T} \sigma_t w_t \ell(\mathbf{z}_t) \right) \right] &= \underset{\mathbf{z}, \sigma}{\mathbb{E}}\Bigg[ \exp \left( 2a \sum_{t=1}^{T-1} \sigma_t w_t \ell(\mathbf{z}_t) \right) \cdot \underset{\mathbf{z}_t, \sigma_t}{\mathbb{E}} \Big[ \exp \left( 2a \sigma_T w_t \ell(\mathbf{z}_T) \right) | \mathbf{z}_{1:T-1}, \sigma_{1:T-1} \Big] \Bigg] \\
    &\leq \underset{\mathbf{z}, \sigma}{\mathbb{E}}\left[ \exp \left( 2a \sum_{t=1}^{T-1} \sigma_t w_t \ell(\mathbf{z}_t) \right) \exp \left( 2a^2 w_t^2 B^2 \right) \right] \leq \exp \left( 2 a^2 B^2 \| w \|_2^2 \right),
\end{align*}
where the equality holds by the law of iterated expectations and the second inequality holds by iteratively applying the previous inequality to $t = 1, 2, \dots, T-1$. 

By optimizing our choice of $a>0$, we conclude that
\begin{align*}
    \underset{Z_{1:T}}{\mathbf{P}}& \left( \sum_{t=1}^{T} w_t \left( \mathbb{E} \left[ \ell \left( Z_t \right) | Z_{1:t-1} \right] - \ell \left( Z_t \right) \right) > \gamma  \right) \leq \exp \left( -a\gamma + 2a^2 B^2 \| w \|_2^2\right) = \exp \left( - \frac{\gamma^2}{8 B^2 \| w \|_2^2} \right).
\end{align*}
Setting the right-hand side to $\delta$ and solving for $\gamma$ completes the proof.
\end{proof}

\section{Discussion}
We use the blocking and decoupling techniques to study the performance of risk-controlling prediction sets, an empirical risk minimization-based formulation of conformal prediction, on a single trajectory of temporally correlated data from an unknown stochastic dynamical system. Using the blocking technique, we identify asymptotically stationary and contractive systems as non-iid data generating processes that RCPS attains iid-like guarantees on. Using the decoupling technique, we characterize the graceful degradation in RCPS guarantees when the underlying system deviates from stationarity and contractivity. Although we focus on RCPS, the blocking and decoupling techniques can be used to analyze the performance of other popular CP algorithms, like split conformal prediction, on non-iid data, as in \cite{oliveira_split_2022, barber_conformal_2023}. 

Furthermore, the decoupling analysis suggests weighted empirical risk minimization can be a useful tool in building adaptive conformal prediction algorithms that perform well in general non-iid data settings. This lends support to heuristic weighing schemes, like exponential smoothing. Developing more principled weighing schemes, potentially by drawing from the online learning literature, can be a fruitful direction for future work.

Because the decoupling technique has also been used to study performance guarantees of online learning algorithms \cite{rakhlin_online_2011}, future work may be able to use it toward a unified analysis of online and offline conformal prediction algorithms, which are currently treated with very different tools. Going further, it may be possible to extend \cite{angelopoulos_online_2024} and prove best-of-both-worlds \cite{mourtada_optimality_2019} and semi-adversarial \cite{bilodeau_relaxing_2023} results for CP, i.e., a single algorithm can attain desirable performance guarantees in both the iid and adversarial data settings, as well various non-iid interpolations thereof. 

\newpage
\bibliographystyle{IEEEtran}
{\small
    \bibliography{references}
}

\addtolength{\textheight}{-12cm}   % This command serves to balance the column lengths

\end{document}